\documentclass[conference]{IEEEtran}
\IEEEoverridecommandlockouts
\usepackage{cite}
\usepackage{amsmath,amssymb,amsfonts}
\usepackage{algorithmic}
\usepackage{graphicx}
\usepackage{textcomp}
\usepackage{xcolor}
\def\BibTeX{{\rm B\kern-.05em{\sc i\kern-.025em b}\kern-.08em
    T\kern-.1667em\lower.7ex\hbox{E}\kern-.125emX}}
    
\usepackage{amsthm}
\newtheorem{lemma}{Lemma}
\usepackage{multirow}
\usepackage{diagbox}
\usepackage{flexisym}
\usepackage{subcaption}
\usepackage{hyperref}
\makeatletter
\newcommand*{\rom}[1]{\expandafter\@slowromancap\romannumeral #1@}
\makeatother

\begin{document}

\title{Unified Spatio-Temporal Modeling for Traffic Forecasting  using Graph Neural Network\\
}

\author{\IEEEauthorblockN{ Amit Roy$^\star$, Kashob Kumar Roy$^{\star}$\thanks{$^\star$Equal Contribution}, Amin Ahsan Ali, M Ashraful Amin and A K M Mahbubur Rahman}
\IEEEauthorblockA{\textit{Artificial Intelligence and Cybernetics Lab},
\textit{Independent University Bangladesh}\\
\{amitroy7781, kashobroy\}@gmail.com, and \{aminali, aminmdashraful, akmmrahman\}@iub.edu.bd}}

%\IEEEauthorblockN{1\textsuperscript{st} Given Name Surname}
% \and
% \and
% \and
% \and
% \and

\maketitle

\begin{abstract}
Research in deep learning models to forecast traffic intensities has gained great attention in recent years due to their capability to capture the complex spatio-temporal relationships within the traffic data. However, most state-of-the-art approaches have designed spatial-only (e.g. Graph Neural Networks) and temporal-only (e.g. Recurrent Neural Networks) modules to separately extract spatial and temporal features. However, we argue that it is less effective to extract the complex spatio-temporal relationship with such factorized modules. Besides, most existing works predict the traffic intensity of a particular time interval only based on the traffic data of the previous one hour of that day. And thereby ignores the repetitive daily/weekly pattern that may exist in the last hour of data. Therefore, we propose a Unified Spatio-Temporal Graph Convolution Network (USTGCN) for traffic forecasting that performs both spatial and temporal aggregation through direct information propagation across different timestamp nodes with the help of spectral graph convolution on a spatio-temporal graph. Furthermore, it captures historical daily patterns in previous days and current-day patterns in current-day traffic data. Finally, we validate our work's effectiveness through experimental analysis\footnote{Code is available at \href{https://github.com/AmitRoy7781/USTGCN}{\color{magenta} github.com/AmitRoy7781/USTGCN}} , which shows that our model USTGCN can outperform state-of-the-art performances in three popular benchmark datasets from the Performance Measurement System (PeMS). Moreover, the training time is reduced significantly with our proposed USTGCN model.
\end{abstract}


\section{Introduction }
% * Two lines about Traffic Forecasting Problems

% * Describe thoroughly Spatial and temporal Relation Handling
% * Historical current time window

In recent years, Intelligent Transportation System (ITS) is being developed in many countries around the world and traffic forecasting lies in the heart of ITS. Traffic intensity is determined by the average speed of vehicles passing through observed road junctions in a traffic network at each time interval and the goal of traffic forecasting is to predict the traffic intensity in near future by observing the traffic data from the past and current time along with the physical traffic network. Accurate forecasting of traffic intensity throughout different parts of the day in a busy city can help the inhabitants to schedule their journeys in an efficient way to avoid traffic jam. % bypassing the crowded path and keeping away from the rush hour when scheduling a trip. 
  Besides, accurate traffic flow prediction is required to recommend time saving paths for drivers and thus crucial for dynamic traffic management. Hence, the problem of traffic forecasting has drawn much interest in artificial intelligence and machine learning community.  

However, the task of traffic forecasting is challenging because there lies a complex spatio-temporal relationship in traffic data as the traffic within a busy city changes heavily in different locations throughout different periods in a day. In order to capture this complex spatio-temporal relationship, several deep learning based approaches are proposed in recent years~\cite{li2017diffusion,guo2019attention,wu2019graph,park2019stgrat,chen2019gated,fang2019gstnet,zhao2019t,ijcai2020-326,xu2020spatial}. As the traffic network can be represented using a graph, graph neural networks~\cite{kipf2016semi,hamilton2017inductive} are used to capture spatial relationships while recurrent neural networks are used to encode the temporal relationships~\cite{yu2017spatio}. To name a few, STGCN~\cite{yu2017spatio} is the first approach that applied graph convolutions along with recurrent units to solve the time series prediction problem in the traffic domain. DCRNN~\cite{li2017diffusion}, on the other hand, used a bidirectional random walk to capture the spatial relations and GRU for temporal dependencies. Again, attention guided spatial and temporal convolutions are performed among recent, daily and weekly components in ASTGCN~\cite{guo2019attention}. Graph Wavenet~\cite{wu2019graph} proposed to learn a self-adaptive adjacency matrix to encode complex spatial relationships in the embedded space while dilated causal convolution layers at different granular levels are used to capture temporal dependency. A very recent work, LSGCN~\cite{ijcai2020-326} captured long-term and short-term dependencies with graph convolution and gated linear unit (GLU). 

\begin{figure}[t]
\centering
\includegraphics[width=1.0\columnwidth]{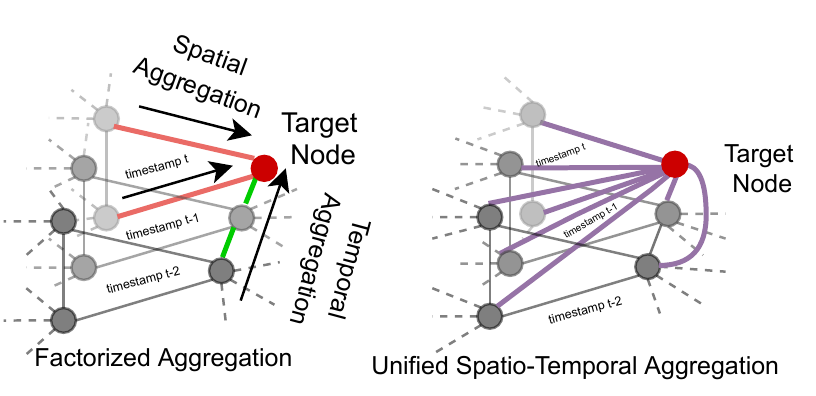}
\caption{ \small Factorized Spatial-only and Temporal-only Aggregation (Left) vs. Unified Spatio-Temporal Aggregation (Right). For a target node in a physical traffic network state-of-the-art approaches capture spatial information from neighbor nodes in each timestamps and aggregate the information for the corresponding node at different timestamps. In contrary, capturing the traffic information for a target node from both spatial and temporal component in a unified manner can learn the inter-relationsip from neighbor nodes at different timestamps more comprehensively.}
\label{fig:motivation_figure}
\end{figure}

% \section{Related Works}

% Although 
Most of the recent works of traffic forecasting follow a typical way of first extracting spatial relationships among different road-junctions/nodes through employing graph neural network on each timestamp traffic network (see in Figure~\ref{fig:motivation_figure}). After that, temporal dependencies are exploited using a 1-D convolution or a recurrent neural network across different timestamp graphs of the physical traffic network. As the physical traffic network (spatial) and continuous traffic data (temporal) components are both related in determining the traffic feature at future timestamps, dealing with them in a factorized manner does not fully serve the purpose to capture the interrelationship.  In Figure~\ref{fig:motivation_figure}, we demonstrate that extracting spatial and temporal information separately and then combining them cannot capture the inherent interrelationship between space-time in traffic data comprehensively. Hence, it is desirable to model both the spatial and temporal dependencies of traffic data in a unified manner. This motivates us to design a Unified Spatio-Temporal Graph Convolution Network (USTGCN) that performs the spatial and temporal aggregation in a unified way as depicted in Figure~\ref{fig:motivation_figure}. Thus, a sophisticated formulation is proposed concatenating the spatial adjacency matrix of all timestamps for a particular time window into a single spatio-temporal adjacency matrix that can propagate the traffic features from different road junctions across different timestamps. USTGCN, later on performs graph convolution with the spatio-temporal graph, which serves the purpose of unified spatio-temporal aggregation. 
% as the traffic within a city changes in different locations across different time of the day
% We have extended the spatial adjacency matrix of a traffic network describing only the spatial connections within a single timestamp to a spatio-temporal adjacency matrix that can convey the traffic feature from different road junctions across different timestamps.

\begin{figure}[b]
\centering
\scalebox{0.95}{
\includegraphics[width=1.0\columnwidth]{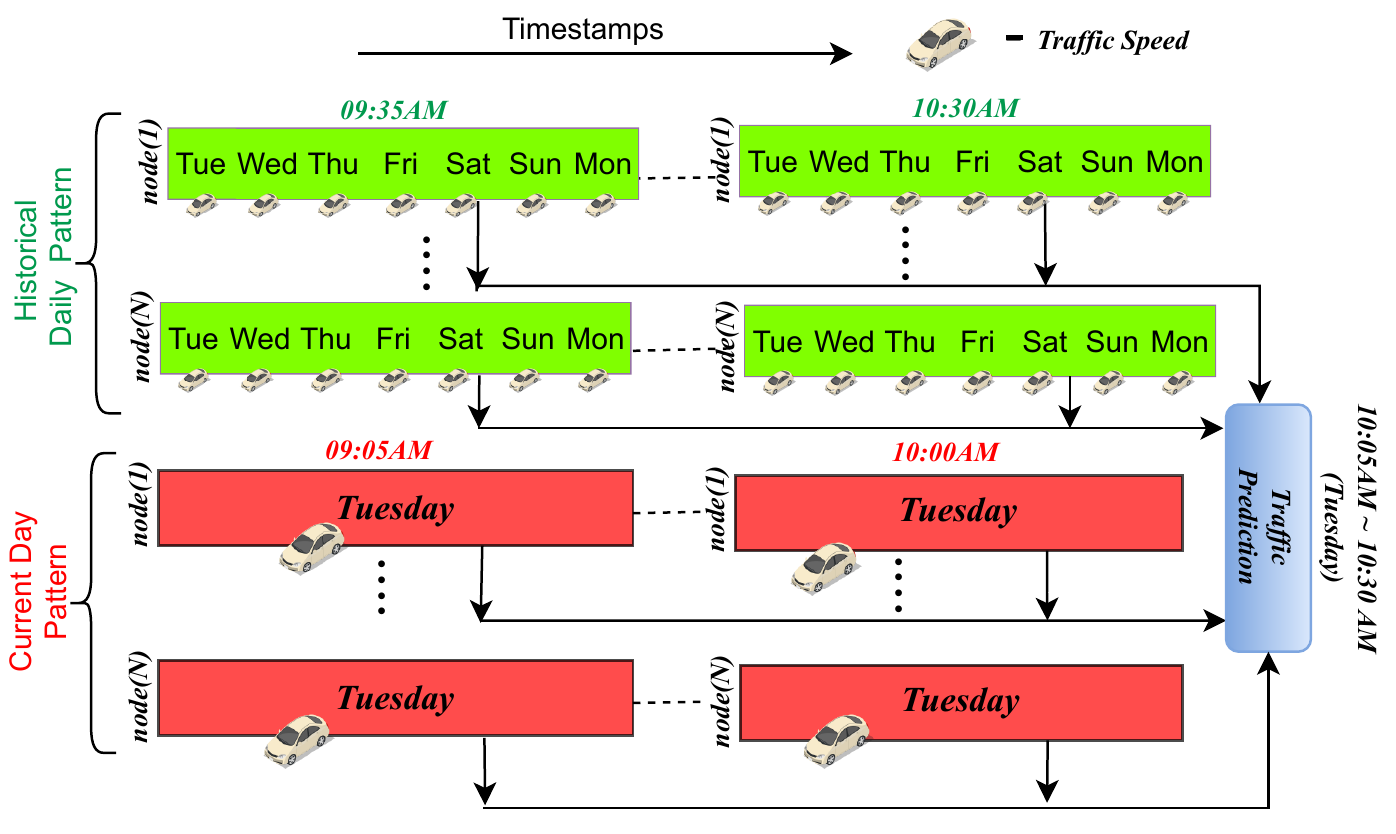}
}
\caption{ \small Daily Patterns in the historical data window from the past days traffic data with Current-day Pattern from the past hour traffic of current day data carry important information to predict traffic intensities. Historical data window consists of last 12 timestamp including prediction window from past seven days' data.}
\label{fig:time_window}
\end{figure}

Another shortcoming of the current state-of-the-art methods is that these models attempted to capture the spatial and temporal dependencies of the current-day pattern only where the current-day pattern is defined as traffic pattern observed in the last hour from the prediction window on the current day. For example, recent approaches aim to learn the pattern in traffic data from 9:05 AM - 10:00 AM to predict the traffic speed at 10:05 AM - 10:30 AM on a particular day. But we argue that only the current-day pattern limits the capability of a traffic forecasting model to predict the traffic intensities in a dynamic metropolitan city whereas the analysis of historical daily patterns can boost the prediction performance as well as the generalization power of the model. As the traffic in a city follows a similar pattern at the same period for different days, it is intuitive that the traffic intensity on Tuesday at 10:05 AM - 10:30 AM will have similarity to the traffic intensity at 10:05 AM - 10:30 AM during each day of the past week which we refer as historical daily pattern throughout this paper. Hence, the historical daily pattern can be represented as the repetitive pattern in traffic data of the last week for a particular time window. For instance, the traffic speed pattern during 10:05 AM - 10:30 AM of each day in the last one week is related in general for predicting the traffic at 10:05 AM - 10:30 AM. To be consistent with the number of timestamps considered in the current-day pattern, in this paper we consider the traffic pattern from 09:30 AM - 10:30 AM of last one week as the historical daily pattern to predict the traffic at 10:05 AM - 10:30 AM. Thus, we analyze the historical daily pattern of traffic speeds during 09:30 AM - 10:30 AM for the last one week as well as the current-day pattern during 9:05 AM - 10:00 AM to predict speeds at 10:05 AM - 10:30 AM on the current day as Figure~\ref{fig:time_window} portrays. Therefore, we propose a novel unified framework to learn the pattern for the last hour on the current day as well as historical daily patterns from traffic data of the previous days.

%In traffic forecasting,
% Additionally, the state-of-the art approaches generally focus on accurately predicting traffic features e.g. average speed, flow, occupancy. Very few of the models are designed to explain the impact of the traffic in neighboring intersections  on the prediction of a particular node. Explainable traffic model  would help the transportation authorities to design an efficient traffic system as well as would help assisting the passengers to schedule their trips avoiding the traffic jams. In image classification, Excitation Backpropagation(EB)~\cite{zhang2018top}, CAM~\cite{zhou2016learning} and Grad-CAM~\cite{selvaraju2017grad} are popular approaches to generate class-discriminative localization maps of the input image. Those localization maps provide the explainability for the model’s classification results on input images. To generate graph-based explanations, CNN-based approaches are utilized to identify which structure of a graph is responsible for graph classification~\cite{pope2019explainability}. To the best of our knowledge, none of the state-of-the art traffic forecasting models have focused on interpretable or explainable traffic forecasting.  In this work, we have exploited a Grad-CAM based approach to generate explainable spatio-temporal heatmaps which will provide explanation behind the prediction of our proposed model. Further, the spectral design of  USTGCN allows us to see the transparent traffic model as well as it facilitates developing the explainable traffic model.

In summary, our key contributions are as follows:
\begin{itemize}
    \item A novel unified spatio-temporal graph convolution network (USTGCN) to capture the complex cross-spacetime dependencies in traffic network data.
    \item A simple but effective approach to extract current-day patterns and historical daily patterns through analyzing traffic data on the current day and last one week respectively.
    % \item Adaptation of explainability method into USTGCN for interpretable traffic forecasting 
    \item With the experimental analysis, we have shown that our unified spatio-temporal model can achieve state-of-the-art performances in three publicly available datasets from the Performance Measurement System (PeMS).
\end{itemize}

\section{Background Study} In the past, various statistical and machine learning techniques such as Auto-Regressive Integrated Moving Average (ARIMA)~\cite{williams2003modeling}, Historical Average (HA), Support Vector Regression (SVR) ~\cite{wu2004travel}, and Kalman filters~\cite{okutani1984dynamic} have been widely used for traffic forecasting. However, in recent years, graph neural networks(GNN) have achieved greater success in modeling real-life traffic. GNNs are able to encode the spatial dependency between neighbor nodes in a graph into their hidden representation by employing different feature aggregation scheme. Graph Convolution Networks (GCN)~\cite{kipf2016semi,defferrard2016convolutional} apply spectral convolutions to learn structural dependency as well as feature information while  GraphSAGE~\cite{hamilton2017inductive} introduced a neighborhood aggregation strategy to preserve the inter-relationship among proximal nodes. 

As GNNs succeeds in learning representations for various downstream machine learning tasks, several recent works have employed graph convolution to learn node representations that can extract spatial relations from the traffic network. STGCN~\cite{yu2017spatio} has modeled spatial and temporal relations using a convolutional network where the spatial graph convolution is applied to extract spatial features in between two temporal gated-convolution with residual connection and bottleneck strategy. The iterative strategy is used for traffic prediction in STGCN where prediction in previous iterations are used for next iterations which accumulates error in prediction. 
The diffusion process is used to model the traffic networks in DCRNN~\cite{li2017diffusion} that captures the spatial relations by using the bidirectional random walks and GRU for temporal dependencies but the random walk based graph convolutional strategy cannot completely capture the spatial relations.
Besides, several recent works\cite{wu2019graph,fang2019gstnet,park2019stgrat} have achieved good performance. To capture the spatio-temporal dependency among nodes in the embedded space, Graph Wavenet~\cite{wu2019graph} learns a self-adaptive dependency matrix where the receptive field increases with the number of layers.  Very recent work LSGCN~\cite{ijcai2020-326} proposes a new graph attention network
called cosAtt and incorporates the cosAtt and GCN into the spatial gated block and linear gated block to iteratively predict future traffic intensity. Handling spatial dependency among different road junctions in the physical traffic network and the temporal variation of traffic data separately limits the the state-of-the-art approaches from encoding the complex  spatio-temporal interrelationship into the learned representation.

\begin{figure*}[!t]
\centering
\scalebox{0.85}{
\includegraphics[width=2.0\columnwidth]{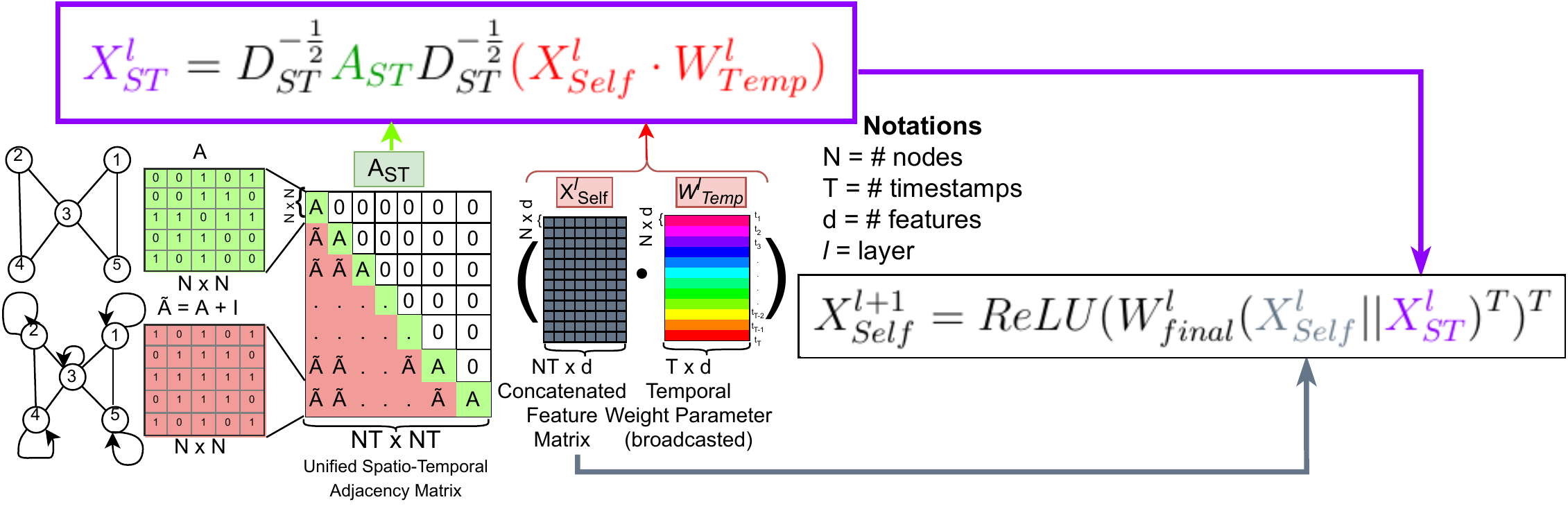}}
\caption{\small Unified Spatio-Temporal Graph Convolutional Network, USTGCN. The unified spatio-temporal adjacency matrix, ${A_{ST}}$ showcases the cross space-time connections among nodes from different timestamps which consists of three types of submatrix: $\textbf{A}$ as diagonal submatrix, $\textbf{\tilde{A}}$ as lower submatrix and $\textbf{0}$ as upper submatrix. ${A_{ST}}$, a lower triangular matrix, facilitates traffic feature propagation from neighboring nodes only from the previous timestamps. The input features of different timestamps at convolution layer ${l}$ are stacked into ${X^{l}_{self}}$ which is element-wise multiplied with broadcasted temporal weight parameter ${W^{l}_{Temp}}$ indicating the importance of the feature at the different timestamp. Afterwards, graph convolution is performed followed by weighted combination of self representation, ${X^{l}_{self}}$ and spatio-temporal aggregated vector, ${X^{l}_{ST}}$  to compute the representation ${X^{l+1}_{self}}$ that is used as input features at next layer, ${l+1}$ or fed into the regression task.
% Unlike traditional factorized spatial-only and temporal-only modules we design unified spatio-temporal aggregation strategy by developing a spatio-temporal graph. 
}
\label{fig:unified_sp_temp_aggregation}
\end{figure*}

\begin{figure*}[!htbp]
\centering
\scalebox{0.85}{
\includegraphics[width=2.0\columnwidth]{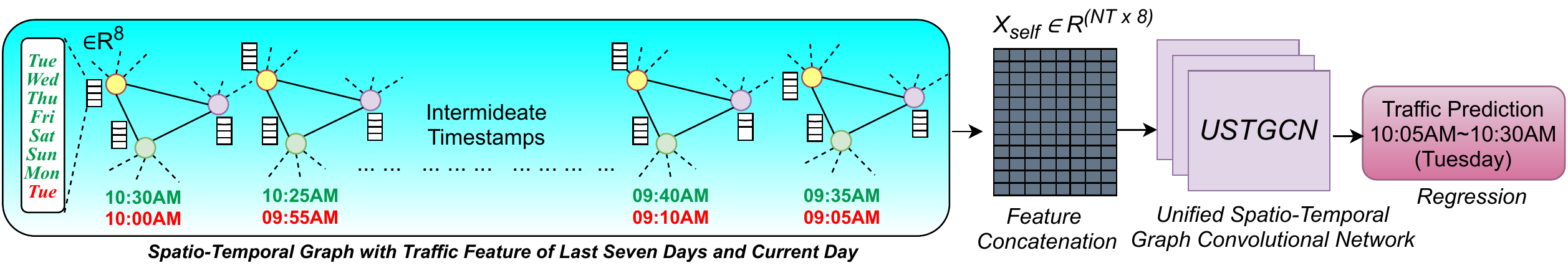}
}
\caption{ \small To learn both daily and current-day traffic pattern, for each node we stack the traffic speeds of the last seven days  (traffic pattern during 09:30 AM - 10:30 AM for the last  week depicted with green color)
along with the current-day traffic pattern for the past hour (traffic speed during 9:05 AM - 10:00 AM on current day i.e. Tuesrday depicted with red color) into the corresponding feature vector. We feed the feature matrix stacked for $N$ nodes in the traffic network across $T = 12$ timestamps to the USTGCN model of \textit{$K$} convolution layers to compute spatio-temporal embedding. Finally, the regression module predicts future traffic intensities by utilizing the spatio-temporal embeddings.}
\label{fig:model_overview}
\end{figure*}

\section{Preliminaries and Problem Setting}
\label{sec:problem_definition}
Forecasting traffic is a regression task where the input is a traffic network that can be represented with the form of a graph $G$ = ( $V$, $E$, $A$) where $V$ and $E$ denotes the set of nodes (road junctions)  and the set of edges (roads) respectively. Let, $N = |V|$. Then, $A \in \mathbb{R}^{N  \times N}$ is an adjacency matrix with $A_{ij} > 0$ if there exists an edge between node $i$ and node $j$, i.e. $(i,j)\in E$ and $0$ otherwise. Also, $A_{ij}$ can denote the edge weight based on the road distance in a weighted traffic network. As the traffic at different nodes changes over time, we can represent the traffic feature of a node $u$ at timestamp $t$ using $X_{u}^{<t>} \in \mathbb{R}^d$ where $d$ is the dimension of traffic feature and the traffic features of all the nodes at timestamp $t$ can be represented using $X^{<t>} \in \mathbb{R}^{N \times d}$. Given a sequence of traffic feature of most recent $T$ timestamps as input  $(X^{<1>},X^{<2>},\ldots \ldots \ldots, X^{<T>})$, 
the task of traffic forecasting is to predict the traffic feature of next $n$ timestamp in the future $(X^{<T+1>},X^{<T+2>}, \ldots \ldots \ldots, X^{<T+n>})$. Here, we refer to recent past $T$ timestamps as data window and next $n$ timestamps as prediction window in the following sections. The traffic feature could be average speed, traffic flow, or road occupancy but following the state-of-art we choose traffic speed as the traffic feature without the loss of generality.

% Most of the popular works in recent times have formulated the traffic forecasting problem as extracting spatial features through graph convolutions from the traffic networks at different timestamps after that capturing the temporal dependencies via recurrent neural networks. However, 

% To encode complex spatio-temporal relationship  of traffic data in the embedded space, it is desired for an algorithm to capture both type of relations in a unified manner. In this section we present a unified approach for spatio-temporal aggregation bringing the concept of spatio-temporal graph that can capture the traffic features from nodes at different timestamps. After that, we applied the spectral graph convolution method to perform both spatio-temporal aggregation in a unified manner. 
\section{Proposed Method}
In this section, we first present the whole architecture of USTGCN that can encode complex spatio-temporal dependencies between nodes of a spatio-temporal traffic network into their embeddings.  Next, we describe the the process of feeding in historical and current-day into USTGCN and lastly the regression module to predict the the traffic intensity in future timestamp using the representations obtained from USTGCN. The brief overview of our proposed model is shown in Figure~\ref{fig:unified_sp_temp_aggregation} and \ref{fig:model_overview}.

\subsubsection{USTGCN: Unified Spatio-temporal Graph Convolution Network} Traditional traffic forecasting frameworks, using spatial-only adjacency matrix $A \in \mathbb{R}^{N \times N}$ for each timestamp $t$ enables the model to capture only the spatial relationships among different junctions of the traffic network at a single timestamp. Therefore, to encode the complex spatio-temporal relationship from different timestamps, we introduce cross-spacetime edges with the form of a spatio-temporal graph consisting of spatial-only $T$ timestamp graphs as follows, 

\begin{equation}
	\centering
	A_{ST} = 
\begin{bmatrix}
	 A        &   0    & 0    & \dots     &   0       \\
	 \tilde{A}       &   A    & 0   & \dots     &  
	 0       \\
			 \tilde{A}       & 		\tilde{A}  &  A & \dots     &   0       \\
			 \vdots  &  			 \vdots   &  		 \vdots  & 	\ddots &   \vdots  \\   
			 \tilde{A}       &   		 \tilde{A} & 	 \tilde{A}      & 	 \dots     &  A       \\
\end{bmatrix}
\in \mathbb{R}^{NT \times NT}
\label{eq:sp_temp_adj}
\end{equation}

% To capture traffic features from $N$ nodes in traffic newtork in the past $T$ timestamps, we define our spatio-temporal graph using an adjacency matrix $A_{ST} \in \mathbb{R}^{NT \times NT}$ in Figure \ref{eq:sp_temp_adj}. 
% We explain this spatio-temporal adjacency matrix $A_{ST}$ in Figure~\ref{fig:unified_sp_temp_aggregation} as a block diagram of shape $T \times T$ where each cell at the $i$th row and the $jth$ column has the shape $N \times N$ that denotes 
% the skip connection from the nodes in timestamp $i$ to the nodes in timestmp $j$, where $i \textless j$. We define $A_{ST}$ as a lower triangular matrix because while traffic forecasting at timestamp $t$ we can access the traffic features from the previous timestamps and not from the future timestamps. 

 In Equation~\ref{eq:sp_temp_adj}, each entry of  $A_{ST}$ is a $N \times N$ matrix as shown in Figure~\ref{fig:unified_sp_temp_aggregation}. It is to be noted that our proposed spatio-temporal adjacency matrix, $A_{ST}$ is a lower triangular matrix because intuitively we can say that traffic intensities of nodes at a particular timestamp $t$ depends on traffic features of nodes from the previous $1$ to $(t-1)$ timestamps. Therefore, for forecasting traffic  at timestamp $t$, an ideal model should aggregate traffic features from the previous timestamps and not from the future timestamps. So, each upper submatrix $[A_{ST}]_{(t, t^{'})}$ = $\mathbf{0}\in \mathbb{R}^{N \times N}$ where $t^{'} > t$, and $t,t^{'} \in \{1,\dots,T\}$. On the other hand, each lower submatrix $[A_{ST}]_{(t, t^{'})}$ = $\tilde{A}$ where $t^{'} < t$ and  $\tilde{A} = A + I$, means that every node at timestamp $t$ aggregates the traffic features of ego (target node)  and neighbor node  from previous $1$ to $(t-1)$ timestamps. And each diagonal submatrix $[A_{ST}]_{(t,t)}$ = A means every node aggregates features from its $1$-hop spatial neighbors at timestamp $t$. We keep the ego (target node) and neighbor node embeddings separate to perform a weighted aggregation explicitly at timestamp $t$. As a result, our model acquires the potential to give  different importance on ego (target node) and on neighbor node features that amplifies its expressiveness. In contrast, if we use $\tilde{A} = A + I$ instead of A, it mixes both embeddings through averaging them which results in less expressive representations of nodes.  
 
 For example, let us consider node $5$ at timestamp 3 in the sample graph in Figure \ref{fig:unified_sp_temp_aggregation} as the target node. $A_{ST}[15:]$ indicates the spatio-temporal connectivity of the target node. Each non-zero entry in $A_{ST}[15:]$ represents the set of edges associated with the target node. $A_{ST}[15:1]$ denotes that node 5 at timestamp 3 is connected with node 1 at timestamp 1. Similarly, $A_{ST}[15:3]$ denotes edge between node 5 at timestamp 3 and node 3 at timestamp 1, $A_{ST}[15:5]$ represents the edge between node 5 at timestamp 3 and node 5 at timestamp 1, and so on. In other words, we can say that the target node (node 5 at timestamp 3) aggregates traffic features from its spatio-temporal neighbor nodes i.e., node $1$, $3$ and $5$ from timestamp $1$, and $2$, and its spatial neighbor nodes i.e., node $1$, and $3$ from timestamp 3. 
%  Again, node $5$ at timestamp $3$ will also aggregate the traffic feature from itself from timestamp $1$ and $2$.
 Consequently, the spatio-temporal adjacency matrix facilitates our model to aggregate spatio-temporal traffic features from informative nodes across previous timestamps as well as spatial traffic features from the neighbor nodes of current timestamp. 

From Section \rom{3}, we can recall that $X^{<t>} \in \mathbb{R}^{N \times d}$ denotes the traffic intensity of all the road junctions of the input traffic network at timestamp $t$.
As depicted in the feature concatenation part of Figure~\ref{fig:model_overview}, we transform the given input features of previous $T=12$  timestamps $(X^{<1>}, X^{<2>}, \ldots \ldots \ldots, X^{<T>}$) where $X^{<t>} \in \mathbb{R}^{N \times d}$ by stacking them in a matrix $X_{Self} \in \mathbb{R}^{NT \times d}$ . After that, we use the adjacency matrix of spatio-temporal graph $A_{ST}$ and the concatenated spatio-temporal feature matrix $X_{Self}$ to perform spectral graph convolution~\cite{kipf2016semi} in Equation~\ref{eq:unified_gcn_1} and~\ref{eq:unified_gcn_2} which serve the purpose of spatio-temporal aggregation in a unified manner. Unlike traditional spectral graph convolution, we introduce a temporal weight parameter $W_{Temp} \in \mathbb{R}^{T \times d}$ which is used to learn the importance of traffic features from different timestamps $1$ to $T$. Note that, $W_{Temp}$ is shared across all the nodes in a particular timestamp hence broadcasted to $\mathbb{R}^{NT \times d}$ to perform element-wise multiplication with the stacked feature matrix $X_{Self}$.

\begin{equation}
\tilde{A}_{ST} = D_{ST}^{-\frac{1}{2}} A_{ST} 
D_{ST}^{-\frac{1}{2}}
\label{eq:unified_gcn_1}
\end{equation}
\begin{equation}
X_{ST}^{l} = \tilde{A}_{ST} (X_{Self}^{l} \cdot W_{Temp}^{l})    
\label{eq:unified_gcn_2}
\end{equation}

Here, $l$ denotes the graph convolutional layer and $D_{ST}[i,i] = \sum_{j=1}^{NT}A_{ST}[i,j]$.

After that, we compute a weighted combination of ego (target node) features $X_{Self}^{l}$ and spatio-temporal aggregated features $X_{ST}^{l}$ with $ReLU$ as the nonlinearity in Equation~\ref{eq:self_sp_temp} to obtain the spatio-temporal representaion for the next layer, $X_{Self}^{l+1}$.
	\begin{equation}
	X_{Self}^{l+1} = ReLU(W_{final}^{l} ( X_{Self}^{l} || 	X_{ST}^{l})^T)^T
	\label{eq:self_sp_temp}
	\end{equation}
	Here, $W_{final} \in \mathbb{R}^{2d \times d}$ is a learnable parameter to assign different importance to the self-representation, $X_{Self}^{l}$ of the current timestamp and spatio-temporal aggregated vector, $X_{ST}^{l}$ from previous timestamps.
	
To increase the receptive field of traffic information, we stack $K$ layers of spatio-temporal convolutions together that enables our model to capture information from the $K$-hop neighborhood for each target node at different timestamps. Figure~\ref{fig:unified_sp_temp_aggregation} shows the overview of unified spatio-temporal convolution.

% To capture the traffic pattern of current day as well as the recent days, we concatenate the traffic feature of last $d$ days along with the current day traffic feature of all the $N$ nodes at previous $T$ timestamp before the prediction window which we denote as $X_{self} \in \mathbb{R^{NT \times d}}$ . 

We designed the spatio-temporal adjacency matrix, $A_{ST}$ as a lower triangular matrix so that the USTGCN network only capture the traffic feature from past timestamps. Hence, the resultant feature matrix from Equation~\ref{eq:unified_gcn_1} should also be a lower triangular matrix to ensure that traffic feature from future timestamps is not included during the aggregation. We show it in the following lemma:

\begin{lemma}
$\tilde{A}_{ST}$ is a lower triangular matrix.
\label{lemma1}
\end{lemma}
\begin{proof}
As we mentioned before, $A_{ST}$ is lower triangular matrix. Also, $ D_{ST}^{-\frac{1}{2}}$ is a diagonal matrix as $D_{ST}$ is a diagonal degree matrix of $A_{ST}$.
\begin{equation}
(DA)[i,j] = \sum_{k=1}^{NT} D_{ST}^{-\frac{1}{2}}[i,k] \times A_{ST}[k,j]
\label{partial_eq_1}
\end{equation}
In Equation~\ref{partial_eq_1}, $(DA)[i,j] = 0$ when $i < j$ (entries above the main diagonal), because either $i \neq k$ or $k < j$ holds,  $D_{ST}^{-\frac{1}{2}}[i,k]$ = 0 or $A_{ST}[k,j]$ = 0. Hence, $(DA)$ is a lower triangular matrix.\\
Again, we rewrite the Equation~\ref{eq:unified_gcn_1} as follows,
\begin{equation}
    \tilde{A}_{ST}[i,j] = \sum_{k=1}^{NT} (DA)[i,k] \times D_{ST}^{-\frac{1}{2}}[k, j]
\label{partial_eq_2}
\end{equation}
Similarly, in Equation~\ref{partial_eq_2}, $\tilde{A}_{ST}[i,j]$ = 0 when $i < j$, because either $i < k$ or $k \neq j$ holds and $(DA)[i,k] = 0$ or $D_{ST}^{-\frac{1}{2}}[k,j] = 0$. Therefore, $\tilde{A}_{ST}$ is a lower triangular matrix.
\end{proof}

Lemma~\ref{lemma1} ensures that USTGCN aggregates spatio-temporal features from previous timestamps only but not from future timestamps.

Untill now, we have described the architecture of USTGCN that convert spatio-temporal graph of physical traffic network and their traffic feature into spatio-temporal representation. 

\subsubsection{Feeding Historical and Current-data into USTGCN} 
We consider the traffic speed of last $P$ days at timestamp $t$ as the historical feature vector of nodes, $X^{<t>}_{H_{u}}$ $\in$ $\mathbb{R}^{P}$, and the traffic speed of current day at timestamp $t$ as current feature of nodes,  $X^{<t>}_{C_{u}}$ $\in$ $\mathbb{R}$ for each node u. As the goal of our model is to preserve the historical traffic information of previous days as well as current day information from last hour traffic data, therefore we concatenate both historical and current feature together i.e., for each node $u$, $X^{<t>}_{u} = X^{<t>}_{H_{u}} || X^{<t>}_{C_{u}}$ and pass into multi-layer USTGCN as input features where $X^{<t>}_{u} \in \mathbb{R}^{P+1}$ as shown in Figure~\ref{fig:model_overview}, 
% Recent-USTGCN employs USTGCN to capture the traffic pattern of the last $P$ days of previous $T$ timestamps from the prediction window as following,
\begin{equation}
    Z_{E} = USTGCN(A, X^{<1>},....,X^{<T>})
\end{equation}
% The motivation behind using rece model is to capture the periodic nature of traffic data from the history of last $P=7$ days.
% \subsubsection{Current-USTGCN:} It only considers  as the feature vector of each node in the network just like the traditional traffic forecasting frameworks. 
% Hence, the current model focuses on the last $T$ timestamps of present day (prediction day) to capture the traffic pattern on the current day.  Similar to Recent-USTGCN, Current-USTGCN finds the current pattern through utilizing USTGCN,
% \begin{equation}
%     Z_{C} = USTGCN(A, X^{<1>}_{C},....,X^{<T>}_{C})
% \end{equation}

\begin{table}[b]
\centering
\scalebox{0.90}{
\begin{tabular}{c|c|c|c}
\hline \textbf{Statistics} & \textbf{PeMSD7} & \textbf{PeMSD4}& \textbf{PeMSD8}  \\ \hline
\textit{\#Nodes}& 228 & 307 & 170 \\ \hline
\textit{\#Edges}& 832 & 340 & 295             \\ \hline
\textit{\#Timestamps}  & 12672 & 16992     & 17856           \\ \hline
\textit{Time Span}     & \begin{tabular}[c]{@{}c@{}}2012/5 - 2012/6\\ (only weekdays)\end{tabular} & 2018/1 - 2018/2 & 2016/7 - 2016/8 \\ \hline
\textit{Time Interval} & \multicolumn{3}{c}{5 minutes}                                                              \\ \hline
\textit{Daily Range}   & \multicolumn{3}{c}{00:00 - 24:00}                                                          \\ \hline
\end{tabular}
}
\caption{\small Description of traffic datasets}
\label{tab:dataset_description}
\end{table}

\begin{table*}[!htbp]
\centering
\scalebox{0.9}{
\begin{tabular}{c|c|ccc|ccc|ccc|ccc}
\hline
\multirow{2}{*}{Datasets} & \multirow{2}{*}{Models} & \multicolumn{3}{c|}{\textit{15 min}}          & \multicolumn{3}{c|}{\textit{30 min}}          & \multicolumn{3}{c|}{\textit{45 min}}          & \multicolumn{3}{c}{\textit{60 min}}          \\ \cline{3-14} 
                          &                         & MAE           & RMSE          & MAPE          & MAE           & RMSE          & MAPE          & MAE           & RMSE          & MAPE          & MAE           & RMSE          & MAPE          \\ \hline

\multirow{6}{*}{PeMSD7}   
% & ARIMA                   & 5.57          & 9.00          & 13.04         & 5.94          & 9.22          & 14.01         & 6.27          & 9.43          & 15.01         & 6.68          & 9.68          & 16.78         \\ 
                          & DCRNN (2018)                   & 2.22          & 4.25          & 5.16          & 3.04          & 6.02          & 7.46          & 3.64          & 7.24          & 9.00          & 4.15          & 8.20          & 10.82         \\  
                          & STGCN (2018)                    & 2.24          & 4.01          & 5.28          & 3.04          & 5.74          & 7.46          & 3.61          & 6.85          & 9.26          & 4.08          & 7.69          & 10.23         \\ 
                          & ASTGCN (2019)                  & 2.85          & 5.15          & 7.25          & 3.35          & 6.12          & 8.67          & 3.70          & 6.77          & 9.73          & 3.96          & 7.20          & 10.53   
                                   \\ 
                          & Graph WaveNet (2019)                  & \underline{2.17}          & \underline{3.87}          & \underline{4.85}          & \underline{2.90}          & \underline{5.40}          & \underline{6.86}          & \underline{3.23}          & \underline{6.29}          & \underline{8.06}          & \underline{3.75}       & \underline{7.02}          & \underline{9.58}
                          \\ 
                        %   & STTN (2019)                  & 2.14          & 4.04          & 5.05          & \textbf{2.70}          & \underline{5.37}          & \underline{6.68}          & \textbf{3.03}          & \underline{6.05}          & \textbf{7.61}          & -       & -          & -
                        %   \\  
                          & LSGCN  (2020)                   & 2.22          & 3.98          & 5.14          & 2.96          & 5.47          & 7.18          & 3.43          & 
                          6.39         & 
                          8.51          & 
                          3.81          & 
                          7.09          & 
                         9.60          \\  
                          
                        %   & ours                    & \textbf{1.97} & \textbf{3.55} & \textbf{4.72} & \textbf{2.64} & \textbf{4.83} & \textbf{6.30} & \textbf{3.17} & \textbf{5.78} & \textbf{8.09} & \textbf{3.63} & \textbf{6.58} & \textbf{9.49} \\ 
                          
                        %   & \textbf{USTGCN (PAKDD)}                    & \textbf{2.03} & \textbf{3.57} & \textbf{4.73} & \textbf{2.70} & \textbf{4.85} & \textbf{6.64} & \textbf{3.15} & \textbf{5.84} & \textbf{7.94} & \textbf{3.67} & \textbf{6.58} & \textbf{9.57} \\
                          
                          & \textbf{USTGCN (ours)}                    & \textbf{2.01} & \textbf{3.48} & \textbf{4.67} & \textbf{2.46} & \textbf{4.43} & \textbf{5.96} & \textbf{2.85} & \textbf{5.07} & \textbf{7.00} & \textbf{3.15} & \textbf{5.54} & \textbf{7.89} \\
                          
                          \hline

\multirow{6}{*}{PeMSD4}   
% & ARIMA                   & 1.90          & 4.87          & 5.11          & 2.12          & 5.24          & 5.21          & 2.43          & 5.63          & 5.46          & 2.79          & 6.22          & 5.62          \\ 
                          & DCRNN  (2018)                  & \underline{1.35} & 2.94 & 2.68 & 1.77 & 4.06 & 3.71 & \underline{2.04} & 4.77 &4.78 &2.26 & 5.28 &5.10 \\ 
                          & STGCN (2018) & 1.47 & 3.01 & 2.92 & 1.93 & 4.21 & 3.98 & 2.26 & 5.01 & 4.73 & 2.55 & 5.65 & 5.39\\  
                          & ASTGCN (2019) & 2.12 & 3.96 & 4.16 & 2.42 & 4.59 & 4.80 & 2.60 & 4.97 & 5.20 & 2.73 & 5.21 & 5.46 \\ 
                          & Graph WaveNet(2019) & \textbf{1.30} & \underline{2.68} &\underline{2.67} & \underline{1.70} & \underline{3.82} & \underline{3.73} &  \underline{1.95} & \underline{4.16} & \underline{4.25} & \underline{2.03} & \underline{4.65} & \underline{4.60}\\
                          & LSGCN (2020)  &1.45 & 2.93 & 2.90 & 1.82 & 3.92 & 3.84 & 2.04 & 4.47 & 4.42 & 2.22 & 4.83 & 4.85 \\  
                          
                        %   & \textbf{USTGCN (PAKDD)}                   & \underline{1.35} & \textbf{2.64} &    \textbf{2.65} & \textbf{1.64} & \textbf{3.31} & \textbf{3.35} & \textbf{1.86} &  \textbf{3.87} &  \textbf{4.02} & \textbf{1.95} & \textbf{4.07} & \textbf{4.16} \\ 
                           
                           & \textbf{USTGCN (ours) }                   & \underline{1.40} & \textbf{2.69} &
                           \textbf{2.81} &
                           \textbf{1.64} & \textbf{3.19} & \textbf{3.23} & \textbf{1.78} &  \textbf{3.64} &  \textbf{3.82} & \textbf{2.03} & \textbf{4.25} & \textbf{4.32} \\
                           \hline
\multirow{6}{*}{PeMSD8}   
% & ARIMA                   & 1.90          & 4.87          & 5.11          & 2.12          & 5.24          & 5.21          & 2.43          & 5.63          & 5.46          & 2.79          & 6.22          & 5.62          \\ 
                          & DCRNN  (2018)                  & {1.17}          & 2.59          & 2.32          & 1.49          & 3.56          & 3.21          & 1.71          & 4.13          & 3.83          & 1.87          & 4.50          & 4.28          \\ 
                          & STGCN (2018)                   & 1.19          & 2.62          & 2.34          & 1.59          & 3.61          & 3.24          & 1.92          & 4.21          & 3.91          & 2.25          & 4.68          & 4.54          \\  
                          & ASTGCN (2019)                   & 1.49          & 3.18          & 3.16          & 1.67          & 3.69          & 3.59          & 1.81          & 3.92          & 3.98          & 1.89          & 4.13          & 4.22          \\ 
                          & LSGCN (2020)                  & \underline{1.16}          & \underline{2.45}          & \underline{2.24}          & \underline{1.46}          & 
                          \underline{3.28}          & 
                          \underline{3.02}          & 
                          \underline{1.66}          & \underline{3.75}          & \underline{3.51}          & 
                          \underline{1.81}          & 
                          \underline{4.11}          & 
                          \underline{3.89}          \\ 
                        %   & \textbf{USTGCN (PAKDD)}                   & \textbf{1.15} & \textbf{2.16} & \textbf{2.07} & \textbf{1.44}          & \textbf{2.86} & \textbf{2.67} & \textbf{1.51} & \textbf{3.12} & \textbf{2.83} & \textbf{1.70} & \textbf{3.44}             & \textbf{3.25} \\ 
                           & \textbf{USTGCN (ours)}                   & \textbf{1.14} & \textbf{2.15} & \textbf{2.07} & \textbf{1.25}          & \textbf{2.58} & \textbf{2.35} & \textbf{1.52} & \textbf{3.01} & \textbf{2.88} & \textbf{1.70} & \textbf{3.27}             & \textbf{3.22} \\
                           \hline
\end{tabular}
} 
\caption{\small Performance comparison of USTGCN with baselines in traffic prediction (\textbf{Bold} = Best, \underline{Underline} = Second Best)}
\label{tab:performance_comparison}
\end{table*}

\noindent

\subsubsection{Regression}
 Once $Z_{E} \in \mathbb{R}^{NT \times d}$ is obtained,  the embeddings of all $T$ timestamps are concatenated and combined into final embedding  $Z_{F}$ as follows,
\begin{equation}
     {Z}_{F} = W_{F} . (Z_{E}^{<1>}\parallel \ldots \parallel Z_{E}^{<T>} )
\end{equation}
where $Z_{E}^{<t>} \in \mathbb{R}^{N \times d}$ represents the spatio-temporal embeddings at timestamp $t$ and $W_{F} \in \mathbb{R}^{Td \times Td} $ is the learnable weight parameter. Finally, ${Z}_{F} \in \mathbb{R}^{N \times Td}$ is passed through a two-layer neural network to predict the traffic speed for each node $u$ and the learnable parameters have been updated by optimizing supervised mean squared error (MSE) as the loss function. Figure~\ref{fig:model_overview} presents the brief overview of whole process.  

\subsubsection{Scalable USTGCN} In real life traffic networks, most of the road junctions are connected with only a few other road junctions compared to the total road junctions present in the network. This results in a \textit{sparse} adjacency matrix of the traffic network. We leverage \textit{sparse} matrix opertions into USTGCN that reduces the space complexity to linear in the number of nodes and edges and enables to execute USTGCN on large graph networks. The total number of nodes in our proposed spatio-temporal graph is $T\times N$ and the total number of edges is $\frac{T(T+1)}{2}(|E| + N)+ T*|E|$ where $|\cdot|$ denotes the cardinality of the edge set, E.

\section{Experimental Analysis}
In this section, we describe the datasets and the experimental setup followed by the elaborate experimental analysis.

\subsubsection{Dataset Description}
To prove the effectiveness of our proposed model, we have performed experiments on three publicly available real-life traffic datasets PeMSD7, PeMSD4, and PeMSD8~\cite{ijcai2020-326} collected by  Caltrans Performance Measurement System (PeMS)~\cite{chen2001freeway}. Those datasets are popular in traffic forecasting research and used for performance comparison in previous works such as STGCN~\cite{yu2017spatio}, ASTGCN\cite{guo2019attention}, and LSGCN\cite{ijcai2020-326}. For partitioning the datasets into training and test set we have followed LSGCN\cite{ijcai2020-326}.
More statistics of these datasets are shown in Table~\ref{tab:dataset_description}

\noindent
\textbf{PeMSD7:} PeMSD7 is traffic data in District 7 of California consisting of the traffic speed of 228 sensors while the period is from May to June in 2012 (only weekdays) with a time interval of 5 minutes. We choose the first month of traffic data as the training set while the rest are used as validation and test set.

\noindent
\textbf{PeMSD4:} The dataset refers to the traffic speed data in San Francisco Bay Area, containing 307 sensors on 29 roads. The time span of the dataset is January-February in 2018 and we choose the first 47 days as the training set and the rest are used as validation and test set. 

\noindent
\textbf{PeMSD8:} This dataset contains the traffic data in San Bernardino from July to August in 2016, with 170 detectors on 8 roads with a time interval of 5 minutes. We select the first fifty days as the training and the rest are used as the validation and test set. 

% PeMSD4 and PeMSD8 dataset contains traffic speed, flow, occupancy information. In our experiments, we adopted the traffic speed as the traffic feature.

\subsubsection{Data Preprocessing}
The regular time interval in all three datasets is 5 minutes which means there exist 288 timestamps in each day. The adjacency matrix of the physical traffic network for PeMSD7 is constructed using a thresholded Gaussian kernel following the recent work LSGCN~\cite{ijcai2020-326},
\begin{equation}
  A_{ij} =
    \begin{cases}
      exp(- \frac{d^{2}_{ij}}{\delta}), i \neq j \text{  and} & exp(- \frac{d^{2}_{ij}}{\delta})  \geq \epsilon \\
      0, \text{otherwise} & 
    \end{cases}       
\end{equation}

while for PeMSD4 and PeMSD8 dataset, the adjacency matrix is constructed as

\begin{equation}
  A_{ij} =
    \begin{cases}
      exp(- \frac{d^{2}_{ij}}{\delta}),  \text{ if $i$ and $j$ are neighbors} \\
      0, \text{otherwise} & 
    \end{cases}       
\end{equation}

where $A_{ij}$ determines edge weight between sensor $i$ and sensor $j$ which is related with $d_{ij}$ (the distance between sensor $i$ and $j$). To control the distribution and sparsity of adjacency matrix $A$, we follow LSGCN~\cite{ijcai2020-326} and set the thresholds $\delta = 0.1$ and $\epsilon = 0.5$

% \begin{table*}[!t]
% \centering
% \scalebox{0.8}{
% \begin{tabular}{c|ccc|ccc|ccc|ccc}
% \hline
% \multirow{2}{*}{Models} & \multicolumn{3}{c|}{\textit{15 min}} & \multicolumn{3}{c|}{\textit{30 min}} & \multicolumn{3}{c|}{\textit{45 min}} & \multicolumn{3}{c}{\textit{60 min}} \\ \cline{2-13} 
%  & MAE & RMSE & MAPE & MAE & RMSE & MAPE & MAE & RMSE & MAPE & MAE & RMSE & MAPE \\ \hline
% \begin{tabular}[c]{@{}c@{}} LSGCN \\ (trained with {PeMSD8}; \\tested on {PeMSD8}) \end{tabular} & 1.16 & 2.45 & 2.24 & 1.46 & 3.28 & 3.02 & 1.66 & 3.75 & 3.51 & 1.81 & 4.11 & 3.89 \\ \hline
% \begin{tabular}[c]{@{}c@{}} USTGCN  \\(trained with  \textbf{PeMSD7};\\ tested on \textbf{PeMSD8})\end{tabular} & 1.16 & \textbf{2.37} & \textbf{2.23} & 1.48 & \textbf{2.78} & \textbf{2.83} & \textbf{1.55} & \textbf{3.29} & \textbf{2.95} & \textbf{1.81} & \textbf{3.54} & \textbf{3.64} \\ \hline
% \end{tabular}
% }
% \caption{Our model's performance on PeMSD8 while trained on PeMSD7}
% \label{tab:trained_PeMSD7_on_PeMSD8_vs_LSGCN}
% \end{table*}

\subsubsection{Experimental Setup}
% All the experiments are performed on a Linux computer (CPU: AMD Ryzen Threadripper 1920X 12-Core Processor, GPU: GeForce RTX 2080 Ti). 
%To locate the best parameters on the validation set we perform a randomized search strategy of hyperparameters. 
In all the experiments, our model considers $T = 12$ timestamps and predicts the traffic feature for the next 15, 30, 45, and 60 minutes. We use the historical traffic data of the prediction window from last $P=7$ days while remaining timestamps are filled from the the past hour timestamps. In addition, the last hour of the present day traffic data are also used to feed as the spatio-temporal feature $X_{self} \in \mathbb{R}^{NT \times 8}$ to USTGCN model. % that is if the traffic data is sampled for every five minutes then we inspect the last hour from the prediction window. 
% In the historical model, the input feature vector of each node comprises the traffic speed of the last seven days while in the current model each node only has one feature value  traffic speed of the current day in the corresponding timestamp. 
We are using three layers in USTGCN for PeMSD7 and PeMSD8 while four layers are used for the experiments of PeMSD4.
%which means each node has the capability to aggregate information from its 3-hop neighborhood. 
We train our model by minimizing Mean Square Error(MSE) as the loss function with ADAM optimizer for 500 epochs. We initially set the learning rate as 0.001 with a decay rate of 0.5 after every 8 epoch up to 24 epoch and set 0.0001 for the rest of the epochs. 
% To evaluate the performance of our model, we choose Mean Absolute Errors (MAE), Root Mean Squared Errors (RMSE), and Mean Absolute Percentage Errors (MAPE) as the evaluation metrics.

% \subsubsection{Evaluation Metrics and Baselines:} To carry out the performance comparison among different models we opt Mean Absolute Errors (MAE), Root Mean Squared Errors (RMSE), and Mean Absolute Percentage Errors (MAPE) as the evaluation metrics. We compare the performance of our model with the following baselines:
% \begin{itemize}
%     \item ARIMA~\cite{williams2003modeling}: Auto-Regressive Integrated Moving Average which is a popular model for the prediction task of time series data.
%     \item DCRNN~\cite{li2017diffusion}: Diffusion Convolution Recurrent Neural Network, this model integrates random walk based graph convolution with recurrent neural network in an encoder-decoder fashion.
%     \item STGCN~\cite{yu2017spatio}: Spatio-temporal Graph Convolutional Network, where gated temporal convolution is combined with graph convolution.
%     \item ASTGCN~\cite{guo2019attention}: Attention-based Spatial-Temporal Graph Convolution Network, attention guided spatial and temporal convolution of recent, daily, and weekly components.
%     \item LSGCN~\cite{ijcai2020-326}: Long Short-term Graph Convolutional Network: Spatial dependency is handled by graph convolution with attention while GLU is proposed to handle temporal dependency.
% \end{itemize}
% \subsection{Experiment Results}

\subsubsection{Evaluation Metrics and Baselines} To carry out the performance comparison among different models we opt Mean Absolute Errors (MAE), Root Mean Squared Errors (RMSE), and Mean Absolute Percentage Errors (MAPE) as the evaluation metrics. We compare the performance of our model with the following baselines:

\begin{itemize}
    % \item ARIMA~\cite{williams2003modeling}: Auto-Regressive Integrated Moving Average which is a popular model for the prediction task of time series data.
    \item DCRNN~\cite{li2017diffusion}: Diffusion Convolution Recurrent Neural Network, this model integrates random walk based graph convolution with recurrent neural network in an encoder-decoder fashion.
    \item STGCN~\cite{yu2017spatio}: Spatio-temporal Graph Convolutional Network, where gated temporal convolution is combined with graph convolution.
    \item ASTGCN~\cite{guo2019attention}: Attention-based Spatial-Temporal Graph Convolution Network, attention guided spatial and temporal convolution of recent, daily, and weekly components.
    \item LSGCN~\cite{ijcai2020-326}: Long Short-term Graph Convolutional Network: Spatial dependency is handled by graph convolution with attention while GLU is proposed to handle temporal dependency.
\end{itemize}

% In  Table~\ref{tab:performance_comparison}, we present the performance comparison of our model with some popular baseline models e.g. DCRNN~\cite{li2017diffusion}, STGCN~\cite{yu2017spatio}, ASTGCN~\cite{guo2019attention}, Graph Wavenet~\cite{wu2019graph}, LSGCN~\cite{ijcai2020-326}  in predicting the traffic speed for the next 15, 30, 45 and 60 minutes.
\subsubsection{Comparison with baselines} 
From the performance comparison in Table~\ref{tab:performance_comparison}, it is easy to observe that our model outperforms all baseline models in both long and short-term predictions for all three evaluation metrics on PeMSD7, PeMSD4, and PeMSD8 except the MAE of PeMSD4 for 15 minute  predictions. 
% Still, it is second best MAE. %The only exception is that the MAE of PeMSD4 for 15 minutea  predictions are close to Graph Wavenet and DCRNN's MAEs but still they are lower. 
The second-best performance is observed for the recent work
Graph Wavenet in dataset PeMSD7, PeMSD4, and by LSGCN in PeMSD8.
Graph Wavenet learns an adaptive adjacency matrix with different granularity whereas LSGCN analyzes long-term and short-term patterns explicitly by employing attention-guided GCN to capture the spatial dependency and GLU to capture the temporal relationship. 
% Capturing long-term and short-term dependencies help LSGCN to perform better than other state-of-the-arts in the long term and short term prediction. 
The results demonstrate those baseline models are still ineffectual to capture complex spatio-temporal dependencies due to their factorized spatial and temporal modules. In contrast, our model is able to capture the complex spatio-temporal relationships through the proposed unified spatio-temporal convolution strategy. In addition to spatio-temporal relations,  our model considers the important historical and current-day pattern by analyzing the data of the last seven days along with current day data. Particularly, including both historical and current day data helps our model in both long-term and short-term prediction with significantly better performance than LSGCN and Graph WaveNet. Furthermore, by introducing $W_{Temp}$ into the unified spatio-temporal aggregation, USTGCN attains the potential to distinguish among different timestamps of different importance which drives the model to achieve significant performance, especially in long-term predictions (30, 45, and 60 minutes) for all the datasets. Because, in long-term predictions, traffic intensities at neighbor nodes of different timestamps in the data window carry different influences on the traffic intensity at a target node in the prediction window.

\begin{table}[!h]
\centering
\scalebox{0.9}{
\begin{tabular}{cccccl}
\hline
\textbf{Model}                                                             & \textbf{\begin{tabular}[c]{@{}c@{}}DCRNN\\ (2018)\end{tabular}} & \textbf{\begin{tabular}[c]{@{}c@{}}STGCN\\ (2018)\end{tabular}} & \textbf{\begin{tabular}[c]{@{}c@{}}Graph WaveNet\\ (2019)\end{tabular}} & \multicolumn{1}{c}{\textbf{\begin{tabular}[c]{@{}c@{}}USTGCN\\ \end{tabular}}} \\ \hline
\begin{tabular}[c]{@{}c@{}}\textbf{Training Time}\\ (s/epoch)\end{tabular} & 1354.75                                                         & 193.45                                                          & 535.51                                                                                       & \textbf{100.91}                                                                       \\ \hline
\end{tabular}
}
\caption{ \small Training time comparison on the PeMSD7 dataset}
\label{tab:triaining_time}
\end{table}

\begin{table}[b]
\centering
\scalebox{0.90}{
\begin{tabular}{c|c|c|c|c|c}
\hline
\textbf{Model}       & \textbf{Metrics} & \textbf{\begin{tabular}[c]{@{}c@{}}15 minutes\end{tabular}} & \textbf{\begin{tabular}[c]{@{}c@{}}30 minutes\end{tabular}} & \textbf{\begin{tabular}[c]{@{}c@{}}45 minutes\end{tabular}} & \textbf{\begin{tabular}[c]{@{}c@{}}60 minutes\end{tabular}} \\ \hline
\multirow{3}{*}{I}   & MAE              & 2.04                                                      & 2.68                                                      & 3.02                                                      & 3.44                                                      \\ \cline{2-6} 
                     & RMSE             & 3.59                                                      & 4.76                                                      & 5.38                                                      & 6.03                                                      \\ \cline{2-6} 
                     & MAPE             & 4.85                                                      & 6.54                                                      & 7.56                                                      & 8.71                                                      \\ \hline
\multirow{3}{*}{A+I} & MAE              & 2.01                                                      & 2.46                                                      & 2.85                                                      & 3.15                                                      \\ \cline{2-6} 
                     & RMSE             & 3.48                                                      & 4.43                                                      & 5.07                                                      & 5.54                                                      \\ \cline{2-6} 
                     & MAPE             & 4.67                                                      & 5.96                                                      & 7.00                                                      & 7.89                                                      \\ \hline
\end{tabular}
}
\caption{ \small Impact of Spatio-Temporal Adjacency Matrix capturing Traffic Data From the Neighbors of Previous Timestamps}
\label{tab:ablation_adj_mat}
\end{table}

\begin{table}[t]
\centering
\scalebox{0.90}{
\begin{tabular}{c|c|c|c|c}
\hline
\begin{tabular}[c]{@{}c@{}}Prediction \\ Hour\end{tabular} & Metrics & $K$=2  & $K$=3  & $K$=4  \\ \hline
\multirow{3}{*}{15 minutes}                                & MAE     & 1.98 & 2.01 & 2.02 \\ \cline{2-5} 
                                                           & RMSE    & 3.47 & 3.48 & 3.46 \\ \cline{2-5} 
                                                           & MAPE    & 4.60 & 4.67 & 4.66 \\ \hline
\multirow{3}{*}{30 minutes}                                & MAE     & 2.47 & 2.46 & 2.48 \\ \cline{2-5} 
                                                           & RMSE    & 4.42 & 4.43 & 4.43 \\ \cline{2-5} 
                                                           & MAPE    & 5.92 & 5.96 & 5.95 \\ \hline
\multirow{3}{*}{45 minutes}                                & MAE     & 2.82 &   2.85   & 2.82 \\ \cline{2-5} 
                                                           & RMSE    & 5.04 &   5.07   & 5.08 \\ \cline{2-5} 
                                                           & MAPE    & 7.01 &   7.00   & 7.13 \\ \hline
\multirow{3}{*}{60 minutes}                                & MAE     & 3.17 & 3.15 & 3.14 \\ \cline{2-5} 
                                                           & RMSE    & 5.62 & 5.54 & 5.52 \\ \cline{2-5} 
                                                           & MAPE    & 7.91 & 7.89 & 7.85 \\ \hline
\end{tabular}
}
\caption{ \small Impact of different number of USTGCN layers for short term and long term traffic prediction on PEMSD7 dataset}
\label{tab:impact_num_layer}
\end{table}

\subsubsection{Training Efficiency} 
Due to the factorized modules (GNN-type spatial-only module followed by RNN-type temporal-only module), the state-of-the-art methods naturally require a large number of parameters. In contrast,
USTGCN performs both the spatial and temporal aggregation in a unified way with requirements of less number of parameters. Besides, its simplicity promotes our model to achieve significant efficiency in terms of the training time of the model. We present the training time comparison of our model USTGCN with other baseline models in Table ~\ref{tab:triaining_time} and observe that USTGCN has faster training time(100.91 second) than all baselines such as STGCN, DCRNN, and GraphWaveNet. Though STGCN is close to our model on the PEMSD7 dataset, USTGCN beats other baselines by a large margin because DCRNN uses complex diffusion convolution with random walks and GraphWaveNet introduces adaptive adjacency matrix learning with different granularities respectively. It demonstrates that either the strategy of using the factorized modules or complex architecture of models could be responsible for large training time whereas spatial and temporal aggregation in a unified GCN manner promotes faster training of USTGCN.

\subsubsection{Ablation Study on Spatio-Temporal Adjacency Matrix} To investigate how USTGCN performs with or without the presence of traffic data from neighbor node from previous timestamps, we have conducted experiments where the traffic data from the neighbor nodes of previous timestamps is captured with $\tilde{A} = A + I$ entries below the main diagonal in the spatio-temporal adjacency matrix $A_{ST}$ (as shown in Equation~\ref{eq:sp_temp_adj}) and  we use $I$ instead of $\tilde{A}$ below the main diagonal in $A_{ST}$ to consider traffic data only from the corresponding node from previous timestamps while ignoring the neighboring nodes. From Table \ref{tab:ablation_adj_mat},  it can be observed that the inclusion of the traffic information from the neighbor nodes at previous timestamps improve the performance of USTGCN on PeMSD7. From our real-life experience we can say that the traffic of a target junction at a particular timestamp not only depends on the traffic of that junction at previous timestamps but also on the traffic of its neighboring junctions from previous timestamps. Our experimental results indicate the same observation in Table~\ref{tab:ablation_adj_mat}. 
% As the traffic flow accumulates from neighboring nodes to the target node, the performance of the model improves if we consider the traffic data from neighbor nodes from previous timestamps.

% \usepackage{multirow}
% \begin{table}[t]
% \centering
% \begin{tabular}{c|c|c|c}
% \hline
%     \multicolumn{2}{c|}{\backslashbox{\textbf{Metrics}}{\textbf{Models}}} & \textbf{I} & \textbf{A+I} \\ \hline
% \multirow{3}{*}{15 minutes}  & MAE   &  2.04 &   2.01      \\ \cline{2-4} 
%                              & RMSE  &  3.59 &  3.48       \\ \cline{2-4} 
%                              & MAPE  &  4.85 &  4.67   \\ \hline
% \multirow{3}{*}{30 minutes}  & MAE   &  2.68 &  2.46       \\ \cline{2-4} 
%                              & RMSE  &  4.76 &  4.43   \\ \cline{2-4} 
%                              & MAPE  &  6.54 &  5.96   \\ \hline
% \multirow{3}{*}{45 minutes}  & MAE   &  3.02 &  2.85   \\ \cline{2-4} 
%                              & RMSE  &  5.38 &  5.07   \\ \cline{2-4} 
%                              & MAPE  &  7.56 &  7.00   \\ \hline
% \multirow{3}{*}{60 minutes}  & MAE   &  3.44 &  3.15   \\ \cline{2-4} 
%                              & RMSE  &  6.03 &  5.54   \\ \cline{2-4} 
%                              & MAPE  &  8.71 &  7.89   \\ \hline
% \end{tabular}
% \caption{Impact of Spatio-Temporal Adjacency Matrix capturing Traffic Data From the Neighbors of Previous Timestamps}
% \label{tab:ablation_adj_mat}
% \end{table}

\begin{table*}[t]
\centering
\scalebox{0.90}{
\begin{tabular}{c|c|c|c|c|c|c|c|c|c|c|c|c}
\hline
\multirow{2}{*}{\begin{tabular}[c]{@{}c@{}}Observed past\\ Traffic  Data (P)\end{tabular}} & \multicolumn{3}{c|}{15 minutes} & \multicolumn{3}{c|}{30 minutes} & \multicolumn{3}{c|}{45 minutes} & \multicolumn{3}{c}{60 minutes} \\ \cline{2-13} 
                                                                                     & MAE       & RMSE     & MAPE     & MAE       & RMSE     & MAPE     & MAE       & RMSE     & MAPE     & MAE       & RMSE     & MAPE     \\ \hline
 3 days                                                                               & 2.00      & 3.48     & 4.64     & 2.73      & 4.76     & 6.58     & 3.13      & 5.46     & 7.76     & 3.54      & 6.14     & 8.82     \\ \hline
7 days &        2.01   &    3.48      & 4.67 &       2.46 &  4.43         &      5.96    &  2.85        &    5.07      &    7.00      &         3.15 &  5.54         &  7.89     \\ \hline
10 days &       2.20    &  3.65 & 5.09         &       2.70    &    4.63      &    6.47      &      2.95     & 5.37         &       7.76   &   3.10        &     5.49     &  7.90        \\ \hline
% 14 days & 2.42 &    3.90      &     5.50     &      3.44     &      5.66    &       8.07   &           &          &          &      4.01     & 6.69         &       9.74   \\ \hline
\end{tabular}
}
\caption{ \small Impact of Traffic Data observed for various number of past days with current day traffic data in Traffic Forecasting}
\label{tab:day_analysis}
\end{table*}

\subsubsection{Impact of number of USTGCN layers in short term and long term prediction} To determine how different number of USTGCN layers affect the performance in long term and short term traffic prediction, we have presented the performance comparison of USTGCN with different number of layer denoted with $K$ in Table ~\ref{tab:impact_num_layer}. From the performance comparison in Table~\ref{tab:impact_num_layer}  we can observe that in short term traffic prediction e.g. 15 minutes the performance of USTGCN with $K=2$ outperforms $K=3$ and $K=4$. However, for one hour prediction $K=4$ performs better than $K=2$ or $K=3$. With increased number of layers, USTGCN starts to aggregate traffic information from neighbors both at same and different timestamps. For short term prediction, the traffic of nearby junctions determine the traffic at next few timestamp of target junction. For this reason, $K=2$ performs better in 15 minutes traffic prediction. On the other hand, for long term prediction the traffic from distant junctions arrive at the target junction which needs to be considered during traffic prediction. This is the reason $K=4$ shows better performance in one hour traffic prediction than  $K=2$ or $K=3$.

\subsubsection{Analysis on Historical Data Observed for Different Number of Days} To evaluate the impact of historical data in traffic prediction we have conducted experiments with the traffic data observed for the various number of past days and reported the result in Table ~\ref{tab:day_analysis} for the PeMSD7 dataset. The results demonstrate that with the current day data, USTGCN performs best for historical data of past 3 days in 15 minutes, past 7 days in 30 and 45 minutes predictions, and for historical data of past 10 days in 60 min prediction. Besides, we can see that all variants outperform state-of-art baseline models. The performance with traffic data observed for the past 3 days is inferior compared to the traffic data observed for the past 7 days except for 15 minute prediction. For short-term prediction, the traffic information of the last few days is quite good. However, with fewer data from past days, the model cannot properly learn the periodicity of traffic data. Again, introducing more historical data may add noise which may harm the model's prediction in short-term predictions. But more historical data still seems helpful for long-term predictions. Again, to process more data from past days requires more training time. From this experimental analysis, it is to be noted that we have to trade-off the number of past days against the training time and performance to design the effective efficient traffic forecasting model.

\begin{table}[b]
\centering
\scalebox{0.90}{
\begin{tabular}{c|c|c|c}
\hline
\multicolumn{2}{c|}{\backslashbox{\textbf{Metrics}}{\textbf{Models}}} & \textbf{Previous Hour} & \textbf{Including Prediction Window} \\ \hline
\multirow{3}{*}{15 minutes}  & MAE   &  2.01    &   2.01      \\ \cline{2-4} 
                             & RMSE  &  3.55     &  3.48       \\ \cline{2-4} 
                             & MAPE  &  4.74     &  4.67   \\ \hline
\multirow{3}{*}{30 minutes}  & MAE   &  2.62     &  2.46       \\ \cline{2-4} 
                             & RMSE  &  4.71     &  4.43   \\ \cline{2-4} 
                             & MAPE  &  6.45     &  5.96   \\ \hline
\multirow{3}{*}{45 minutes}  & MAE   &  3.02     &  2.85   \\ \cline{2-4} 
                             & RMSE  &  5.35     &  5.07   \\ \cline{2-4} 
                             & MAPE  &  7.54    &  7.00   \\ \hline
\multirow{3}{*}{60 minutes}  & MAE   &  3.40    &  3.15   \\ \cline{2-4} 
                             & RMSE  &  6.40    &  5.54   \\ \cline{2-4} 
                             & MAPE  &  8.64    &  7.89   \\ \hline
\end{tabular}
}
\caption{ \small Experimental Analysis on Historical data window: Traffic Data from Previous Hour vs Last 12 timestamps including Prediction Window}
\label{tab:time_window_analysis}
\end{table}

\subsubsection{Analysis on Historical Data from Different Time Window} Throughout all the experiments we consider the past 12 timestamps from the prediction window as current-day data window and for historical data, the last $12$ timestamps including the prediction window of last $P$ days (as in Figure~\ref{fig:time_window} we described in the introduction section).  To evaluate the impact of the different possible data time window, we experiment on another possible traffic data window for historical data which consists of 12 timestamps from the past hour of the prediction window.
% and with the traffic data of prediction hour. 
% As our model considers 12 timestamps as the input we fill the absence of traffic data in the prediction of 15 minutes(3 timestamps), 30 minutes(6 timestamps), and 45 minutes(9 timestamps) from the past hour of the traffic data. 
From Table~\ref{tab:time_window_analysis}, we can observe that with the traffic data of the prediction window in the historical data window, the performance of USTGCN improves in all three metrics. Intuitively, it is to say that to capture the traffic patterns from past days' data, the data window including the prediction window provides more informative patterns than the data window from the past hour of the prediction window. 

\section{Conclusion}
Our proposed Unified Spatio-Temporal Graph Convolution Network (USTGCN) effectively captures complex spatio-temporal dependencies between nodes across different timestamps without the requirement of a factorized spatial-only and temporal-only module. Both historical and current features from past days and current-day data enable our model to encode historical daily pattern as well as current-day traffic pattern into the hidden representation of nodes. Besides, due to its simplicity, the model gains significant efficiency in terms of training time. Moreover, the extensive experiential analysis on several real-life datasets verifies the effectiveness and efficiency of our model.

\section{Acknowledgements}
This project is supported by a grant from the Independent University Bangladesh and ICT Division of Bangladesh Government.
\bibliographystyle{IEEEtran}
\bibliography{ReferencesTraffic}   

\end{document}